\documentclass[conference]{IEEEtran}
\IEEEoverridecommandlockouts

\usepackage{cite}
\usepackage{amsmath,amsthm,amssymb,amsfonts}
\usepackage{algorithmic}
\usepackage{graphicx}
\usepackage{textcomp}
\usepackage{xcolor}
\usepackage{verbatim}
\usepackage{latexsym}
\usepackage{nicefrac}
\usepackage{dsfont}%
\usepackage{ulem}
\usepackage{enumerate}
\usepackage{mdwlist}
\usepackage[american]{babel}
\usepackage[utf8]{inputenc}
\usepackage{soul}
\usepackage{tikz}
\usepackage[hidelinks]{hyperref}

\def\BibTeX{{\rm B\kern-.05em{\sc i\kern-.025em b}\kern-.08em
    T\kern-.1667em\lower.7ex\hbox{E}\kern-.125emX}}

    \newcommand{\C}{\mathbb{C}}
    \newcommand{\N}{\mathbb{N}}
    \newcommand{\R}{\mathbb{R}}

    \newcommand{\Bcal}{\mathcal{B}}
    \newcommand{\Ccal}{\mathcal{C}}
    \newcommand{\Fcal}{\mathcal{F}}
    \newcommand{\Hcal}{\mathcal{H}}
    \newcommand{\Lcal}{\mathcal{L}}
    \newcommand{\Mcal}{\mathcal{M}}
    \newcommand{\Ncal}{\mathcal{N}}
    
    \newcommand{\Pcal}{\mathcal{P}}
    
    \newcommand{\Rcal}{\mathcal{R}}

    \newcommand{\Mfrak}{\mathfrak{M}}

    \newcommand{\layer}{{(\ell)}}
    \newcommand{\player}{{(\ell-1)}}

    \newcommand{\dmul}{\:\mathrm{d}\mu^\layer}
    \newcommand{\dmuG}{\:\mathrm{d}\mu_{G}}
    \newcommand{\muGhat}[1]{\ensuremath{\mu_{\widehat{G}}\left(\text{}#1\text{}\right)}}
    \newcommand{\dmuGhat}{\:\mathrm{d}\mu_{\widehat{G}}}
    
    \newcommand{\norm}[1]{\ensuremath{\left\lVert\text{}#1\text{}\right\rVert}}
    \newcommand{\ip}[1]{\ensuremath{\left<\text{} \, #1 \, \text{}\right>}}
    \newcommand{\set}[1]{\ensuremath{\left\{\text{}#1\text{}\right\}}}

    \newtheorem{assumption}{Assumption}[section]

    \newtheorem{remark}[assumption]{Remark}
    
    \newtheorem{proposition}[assumption]{Proposition}
    \newtheorem{theorem}[assumption]{Theorem}
    \newtheorem{corollary}[assumption]{Corollary}

    \DeclareMathOperator{\id}{id}

    \DeclareMathOperator{\supp}{supp}

    \newcommand\copyrighttext{%
    	\footnotesize \textcopyright 2025 IEEE. Personal use of this material is permitted. Permission from IEEE must be obtained for all other uses, in any current or future media, including reprinting/republishing this material for advertising or promotional purposes, creating new collective works, for resale or redistribution to servers or lists, or reuse of any copyrighted component of this work in other works.
    }
    
    \newcommand\copyrightnotice{%
    	\begin{tikzpicture}[remember picture,overlay]
    		\node[anchor=north,yshift=-10pt] at (current page.north) {\fbox{\parbox{\dimexpr\textwidth-\fboxsep-\fboxrule\relax}{\copyrighttext}}};
    	\end{tikzpicture}%
    }
    
    \newcommand\citationtext{%
    	\footnotesize M. Getter, ``Digital implementations of deep feature extractors are intrinsically informative'', 2025 International Conference on Sampling Theory and Applications (SampTA), 2025, pp. 1-6, 
    	DOI: \href{https://doi.org/10.1109/SampTA64769.2025.11133552}{10.1109/SampTA64769.2025.11133552}.
    }
    
    \newcommand\citationnotice{%
    	\begin{tikzpicture}[remember picture,overlay]
    		\node[anchor=south,yshift=10pt] at (current page.south) {\fbox{\parbox{\dimexpr\textwidth-\fboxsep-\fboxrule\relax}{\citationtext}}};
    	\end{tikzpicture}%
    }

\begin{document}

\title{Digital implementations of deep feature extractors are intrinsically informative
\thanks{The author acknowledges funding by the Deutsche Forschungsgemeinschaft (German Research Foundation)---Project number 442047500 through the Collaborative Research Center ”Sparsity and Singular Structures” (SFB 1481).
}
}

\author{\IEEEauthorblockN{Max Getter}
\IEEEauthorblockA{\textit{Chair for Geometry and Analysis} \\
\textit{RWTH Aachen University}\\
getter@mathga.rwth-aachen.de}%
}

\maketitle
\copyrightnotice
\citationnotice
\vspace*{-1em}

\begin{abstract}
  Rapid information (energy) propagation in deep feature extractors is crucial to balance computational complexity versus expressiveness as a representation of the input. We prove an upper bound for the speed of energy propagation in a unified framework that covers different neural network models, both over Euclidean and non-Euclidean domains. Additional structural information about the signal domain can be used to explicitly determine or improve the rate of decay. To illustrate this, we show global exponential energy decay for a range of 1) feature extractors with discrete-domain input signals, and 2) convolutional neural networks (CNNs) via scattering over locally compact abelian (LCA) groups.
\end{abstract}


\section{Introduction}\label{sec: Introduction}

Deep feature extractors (feature maps) are supposed to produce meaningful representations of the input data for a machine learning task. %
It is generally hard to predict if a specific feature is relevant to the task at hand \cite{bengio2013representation}. To assess the power of a feature extractor, it is hence necessary to rely on more generic (i.e., task-unspecific) criteria. These include the ability of the feature extractor to encode abstract concepts---resulting in (approximate) invariance to local changes of the input---and its ability to disentangle the factors that are responsible for variation within the data \cite{bengio2013representation,hinton2006reducing, mallat2016understanding,ranzato2007unsupervised}. Both criteria are related by the key component of information retention, which leaves feature engineers with the challenging task of balancing the discriminative power of the feature extractor (i.e., the ability to disentangle the driving factors for variance within the data) and the potential loss of information (as a result of transforming the input). In this contribution, we focus on the premise that deep feature extractors should contain most of the (relevant) information about their input signals, which is expressed by the aggregate energy content of the features.  %

\subsection{Prior work}\label{sec: Prior work}

Information retention is a key %
component of successful feature extraction. It has hence already been addressed by prior works---mainly in the context of scattering CNNs, both over Euclidean domains \cite{mallat2012group,waldspurger2017exponential,wiatowski2017energy,wiatowski2017topology,czaja2019analysis,fuhr2025energy} and graphs \cite{zou2020graph,perlmutter2023understanding}. While the question of energy decay seems to be reasonably settled for scattering CNNs over Euclidean domains, the literature is rather sparse on energy decay in more general feature extractors, in particular, scattering CNNs over non-Euclidean domains. %

\subsection{Our contributions}\label{sec: Our contributions}

Motivated by the construction of scattering CNNs \cite{mallat2012group,wiatowski2017mathematical,chew2024geometric}, we provide a unifying framework to study the properties of general deep feature extractors acting on signals defined over arbitrary measure spaces. The framework is particularly well suited to quantitatively analyze their %
stability and information propagation. 

The existing results regarding energy decay are %
conceptually closely related to each other. %
We exploit this insight to establish a generic rate for energy decay in general deep feature extractors---%
building on \cite[Proposition 3.3]{zou2020graph}, where the authors show that energy decay is exponential with increasing network depth for certain graph CNNs. As a corollary, we obtain exponential energy decay for general deep feature extractors if the underlying measure space does not contain sets of arbitrarily small positive measure. 

We extend ideas from \cite{mallat2012group,czaja2019analysis} to LCA group scattering and combine them with sumset-estimates based on \cite{ruzsa1999analog}. %
This yields exponential energy decay essentially if the frequency supports of the filters are uniformly bounded. As a byproduct, we thereby validate the experimental observations from \cite{bruna2013invariant,anden2011multiscale} regarding the energy distribution in scattering CNNs. More generally, we conclude that digital implementations of deep %
feature extractors are intrinsically informative.

\section{Information retention in general feature extractors}\label{sec: Information Retention in General Feature Extractors}

\subsection{A general model for feature extraction}\label{sec: A general model for feature extraction}

For every $\ell \in \N_0$, let $\Hcal^\layer:=L^2(\Mcal^\layer)$ be the Lebesgue-space\footnote{If there is no ambiguity, we omit the Hilbert-space index to denote the corresponding inner product $\ip{\cdot,\cdot}$ or norm $\norm{\cdot}$.} of square-integrable complex-valued functions over a measure space $\Mcal^\layer:=(M^\layer,\Mfrak^\layer,\mu^\layer)$.
Set $\Hcal:=\Hcal^{(0)}$. For a broad range of both finite- or infinite-depth neural networks, the modules for the $\ell$-th layer, $\ell\in \N$, %
of such a neural network can be described by %
an at most countable family of bounded linear operators $\Lcal^\layer \subseteq \Bcal(\Hcal^\player,\Hcal^\layer)$, which forward the information %
from layer $\ell-1$ to layer $\ell$, a 
nonexpansive\footnote{That is, $1$-Lipschitz, for all $f,g\in \Hcal^\layer$, $\norm{\sigma^\layer f-\sigma^\layer g}\leq \norm{f-g}$.} %
map $\sigma^\layer:\Hcal^\layer \to \Hcal^\layer$ with $\sigma^\layer(0)=0$, which allows to introduce additional complexity to the model%
, and a bounded linear operator $A^\layer \in \Bcal(\Hcal^\layer)$, which allows to generate an output at depth $\ell$. Likewise, $A^{(0)}\in \Bcal(\Hcal)$ allows to produce a first output at layer depth zero. For any input signal $f \in \Hcal$, the neural network associated with these modules generates the output
\begin{align*}%
  Sf:=\set{S[p]f~:~ \ell\geq 0, ~p \in \Pcal^\layer}, %
\end{align*}
where $\Pcal^\layer:=\Lcal^{(1)}\times \cdots \times \Lcal^{(\ell)}$ for $\ell\in \N$, $\Pcal^{(0)}:=\{p_e\}$ contains the unique empty path $p_e$ of length $\ell=0$ with associated operator $S[p_e]f:=A^{(0)}f$, and for any path $p=(L^{(1)},\ldots,L^{(\ell)})\in \Pcal^\layer$ of length $\ell\in \N$, 
\begin{align*}
  S[p]f:=A^{(\ell)} U[p]f, \quad
  U[p]f:=\sigma^{(\ell)} L^{(\ell)} \ldots \sigma^{(1)}L^{(1)} f. %
\end{align*}

\subsection{Energy conservation}\label{sec: Energy conservation}

The following condition is key to many desirable properties of neural networks employed as feature extractors, and is always assumed below.  %
\begin{assumption}\label{ass: Frame analog}
  For every $\ell\in \N$ and for all $h \in \Hcal^\player$,
  \begin{align*}
    \norm{A^\player h}^2+\sum_{L\in \Lcal^\layer} \norm{L h}^2 \leq \norm{h}^2. 
  \end{align*}
\end{assumption}

Let us briefly comment on the flexibility of the model. 
  Both finite- and infinite-depth neural networks can be modeled by the framework, for if a network consists only of finitely many layers $\ell=0,\ldots,D$, one can simply take, for all $\ell>D$, $\Lcal^\layer=\emptyset$ and $A^\layer=\sigma^\layer=\id_{\Hcal^{(D)}}$.

  The action of the $\ell$-th layer of a feedforward neural network \cite{rumelhart1986learning} is typically described by $f^\layer=\sigma(W^\layer f^\player)$, where $W^\layer\in \C^{d^\layer\times d^\player}$ are the (learned) weights including a potential bias term, $\sigma$ is a pointwise nonlinearity, and $f^\layer$ is the output of the $\ell$-th layer. Clearly, this fits into our framework, by setting $\Hcal=\C^{d^{(0)}}$, $\Hcal^\layer=\C^{d^\layer}$, $\Lcal^\layer=\set{W^\layer}$, $\sigma^\layer=\sigma$, for $\ell=1,\ldots,D$, as well as $A^\layer=0_{d^\layer\times d^\layer}$ for $\ell=0,\ldots,D-1$, and $A^{(D)}=\id_{\C^{d^{(D)}}}$. Assumption \ref{ass: Frame analog} then boils down to the requirement that the singular values of $W^\layer$ are $\leq 1$. Linear feedforward neural networks \cite{bah2022learning} are a special instance thereof, taking pointwise identities for $\sigma$. 
  
  Section \ref{sec: Information retention in scattering CNNs over LCA groups} studies scattering CNNs in the above framework. The framework description is the same for more general CNNs (learned or not). Pooling can be incorporated by adjusting either the nonlinearities or the linear operators $\Lcal^\layer$. %
  
  The next proposition, which collects several properties of $S$ that are based on straightforward generalizations of \cite{wiatowski2017mathematical,chew2024geometric}, makes apparent why Assumption \ref{ass: Frame analog} is natural for deep feature extractors. %

\begin{proposition}\label{prop: energy conservation}
  We have, for every $f\in \Hcal$ and every $N\in \N$,
  \begin{align}\label{eq: upper bound for total energy of a layer}
    \sum_{p \in \Pcal^{(N)}} \norm{S[p]f}^2 +  W_{N+1}(f) \leq W_N(f),
  \end{align}
  where %
  $W_N(f):=\sum_{p\in \Pcal^{(N)}} \norm{U[p]f}^2$. Thus, %
  \begin{align}\label{eq: energy conservation}
    \sum_{n=0}^{N-1} \sum_{p\in \Pcal^{(n)}} %
    \norm{S[p]f}^2 + W_N(f)\leq \norm{f}^2.
  \end{align}
  Further, $S:\Hcal \to \bigoplus_{\ell=0}^\infty \bigoplus_{p\in \Pcal^\layer} \Hcal^\layer$ is nonexpansive with $S(0)=0$, hence norm-decreasing. 
  \end{proposition}

\begin{remark}
  Assumption \ref{ass: Frame analog} can be relaxed by replacing the right-hand side (RHS) of the inequality with the weaker condition $\leq B^\layer \norm{h}^2$ for a constant $B^\layer>1$. Such a relaxation would be necessary  to model skip connections (as required in, e.g., residual neural networks). %
  However, upper bounds on $W_N(f)$ are more meaningful if $B^\layer\leq 1$. Assuming $B^\layer\leq 1$ also entails that %
  the resulting feature extractor $S$ is nonexpansive, which provably guarantees other desired properties of $S$, such as its stability with respect to small deformations of the input \cite{mallat2012group}.
  Likewise, for scattering CNNs it is common to require 
  an additional lower frame condition of the type 
    $c^\layer \norm{h}^2 \leq \norm{A^\player h}^2+\sum_{L\in \Lcal^\layer} \norm{\sigma^\layer L h}^2$,
  $c^\layer>0$ independent of $h\in \Hcal^\player$. 
  In these cases, energy conservation holds in the sense of $c\norm{f}^2\leq \norm{Sf}^2\leq \norm{f}^2$, where $c=\prod_{\ell=1}^\infty c^\layer\geq 0$, cf. \cite{wiatowski2017energy}.

\end{remark}

Starting from \eqref{eq: upper bound for total energy of a layer}, a telescoping argument yields
\begin{align*}
  \sum_{\ell=N}^\infty \sum_{p\in \Pcal^\layer} \norm{S[p]f}^2 
  \leq \sum_{\ell=N}^\infty \left(W_\ell(f)-W_{\ell+1}(f)\right) \leq W_N(f). %
\end{align*} 
The infinite sum $\sum_{\ell=N}^\infty \sum_{p\in \Pcal^\layer} \norm{S[p]f}^2$ provides a natural measure to quantify the information contained in the layers of depth $\geq N$. This quantity is bounded above by $W_N(f)$. Assuming that $Sf$ is nontrivial, truncating the (possibly infinitely deep) feature extractor after the first $N$ layers looses at most $W_N(f)/\norm{Sf}^2$ percent of the total energy that is distributed across the entire network, suggesting that a fast decay of this quantity is desirable (see also \cite{wiatowski2017energy}). Rapid energy decay across layers is also provably closely related to the stability of scattering CNNs with respect to pertubations of the input \cite[Corollary 2.15]{mallat2012group}, \cite[Theorem 5]{chew2024geometric}.

On the basis of Assumption \ref{ass: Frame analog}, we derive a generic upper bound for $W_N(f)$, for arbitrary $f\in \Hcal$ and $N\in \N$. Let us start by outlining the general proof strategy, the starting point of which is similar to the one used in previous works on energy propagation in scattering CNNs, e.g. \cite{czaja2019analysis}. %
By Assumption \ref{ass: Frame analog}, %
\begin{align}\label{eq: Upper bound on W_{N+1}(f) by Assumption 1}
    W_{N+1}(f)&= \sum_{p \in \Pcal^{(N)}}\sum_{L\in \Lcal^{(N+1)}} \norm{\sigma^{(N+1)}LU[p]f}^2 \nonumber\\
    &\leq \sum_{p \in \Pcal^{(N)}} \left(\norm{U[p]f}^2-\norm{A^{(N)}U[p]f}^2\right).
\end{align}
We introduce 
\begin{align}
  \iota_N :&= \inf_{h \in \Hcal, p \in \Pcal^{(N)}:U[p]h\neq 0} \frac{\norm{A^{(N)}U[p]h}^2}{\norm{U[p]h}^2} \nonumber\\
  &=\inf_{0\neq g\in\Rcal^{(N)}}\norm{A^{(N)}\left(\frac{g}{\norm{g}}\right)}^2, \label{def: iota_N}
\end{align}
where $\Rcal^{(N)}:=\bigcup_{p \in \Pcal^{(N)}} \Rcal(U[p])\subseteq \Hcal^{(N)}$, 
$\Rcal$ referring to the range of the operators. Using this to estimate \eqref{eq: Upper bound on W_{N+1}(f) by Assumption 1}, we obtain %
\begin{align}
    W_{N+1}(f)%
    &\leq \sum_{p \in \Pcal^{(N)}} \left(1-\iota_N\right) \norm{U[p]f}^2 \nonumber\\
    &= W_N(f) \cdot \left(1-\iota_N\right) \nonumber\\
    &\leq W_1(f) \cdot \prod_{\ell=1}^{N}(1-\iota_\ell)\nonumber\\
    &\leq \left(\norm{f}^2-\norm{A^{(0)}f}^2\right) \cdot \prod_{\ell=1}^{N}(1-\iota_\ell). \label{eq: introduction of iota_N}
\end{align}

Describing the range $\Rcal(U[p])$ of the operator $U[p]$ for any path $p$ of length $\geq 2$ is generally a hard problem, and so is the explicit calculation of $\iota_N$. %

\subsection{A generic rate for energy decay}\label{subsec: generic rate}
  We reduce the question of energy decay to an optimization problem involving the point spectrum of the operators ${A^\layer}^\ast A^\layer$, $\ell \in \N$. To this end, let  %
  \begin{align*}
    E^\layer:=\set{(\lambda,\eta)\in [0,1]\times \Hcal^\layer \middle|{A^\layer}^\ast A^\layer \eta = \lambda \eta,  \norm{\eta}=1}.
  \end{align*}
  Suprema (and infima) over empty sets are to be read as zero.
\begin{theorem}\label{thm: energy decay in the general setting}
  We have, for every $f\in \Hcal$ and every $N\in \N$,\vspace*{-0.3mm}
  \begin{align*}
    W_N(f)\leq \left(\norm{f}^2-\norm{A^{(0)}f}^2\right) \cdot \prod_{\ell=1}^{N-1} \left(1-C^\layer\right),
  \end{align*}
  \vspace*{-0.3mm} where $C^\layer:=1$ if $\Rcal^\layer=\{0\}\subseteq \Hcal^\layer$ and else %
  \begin{align}\label{def: constant C^ell}
    C^\layer:=\inf\limits_{0 \neq g\in \Rcal^\layer}
    ~\sup\limits_{(\lambda,\eta)\in E^\layer}~ \left|\ip{g/\norm{g},\eta}\right|^2 \cdot \lambda \geq 0.
  \end{align}
\end{theorem}

\begin{proof}
  The proof is given in Appendix \ref{sec: app - thm: energy decay in the general setting}. 
\end{proof}

\begin{remark}\label{rem: generic decay rate}
  Note that $\lambda, C^\layer \in [0,1]$, by Assumption \ref{ass: Frame analog}.

  While $E^\layer=\emptyset$ may occur in the very general case, we have $E^\layer\neq\emptyset$ in many practically relevant settings. In fact, the setting is inspired by that of scattering CNNs over general measure spaces \cite{chew2024geometric}, where $A^\layer=H^t$ is the diffusion operator defined in \cite[Equ. (8),(10)]{chew2024geometric} for some $t>0$, so $E^\layer=\{(g(\lambda_k)^{2t},\varphi_k)~|~k\in I\}$ in the notation of \cite{chew2024geometric}. %
  
  Finally, we note that there are feature extractors for which the generic bound obtained by this approach is sharp, cf. \cite{zou2020graph}.
\end{remark}

Theorem \ref{thm: energy decay in the general setting} may---at first glance---seem like an artificial reduction to another difficult problem, namely bounding the constant $C^\layer$ from below. However, at least for some types of specific feature extractors (respectively, their underlying signal spaces), it is actually easy to do so, as we demonstrate below. Indeed, Theorem \ref{thm: energy decay in the general setting} generalizes the ideas of \cite[Proposition 3.3]{zou2020graph} that led the authors to conclude exponential energy decay for certain graph CNNs. Likewise, we conclude exponential energy decay for general feature extractors if the measure space does not contain sets of arbitrarily small positive measure, which recovers \cite[Proposition 3.3]{zou2020graph} as a special case of our next corollary\footnote{It can also be seen to recover \cite[Theorem 3.4]{perlmutter2023understanding} with a minor adjustment in our setting.}.  %

  \begin{corollary}\label{cor: generic rate for discrete measure spaces}
     Assume that $C_{\Mcal^\layer}:=\inf_m ~ \mu^\layer(m)>0$, where the infimum is taken over all sets $m\in\Mfrak^\layer$ for which $\mu^\layer(m)>0$. Suppose %
     further that $\sigma^\layer f\geq 0$ ($\mu^\layer$-a.e.) for all $f \in \Hcal^\layer$ and %
    that there exists a pair $(\lambda,\eta)\in E^\layer$ such that $\lambda^{1/2} \eta \geq C_{A^\layer}^{1/2}$ holds ($\mu^\layer$-a.e.) for a constant $C_{A^\layer}>0$. %
    Then, we have $C^\layer=1$ if $\Rcal^\layer=\{0\}\subseteq \Hcal^\layer$ and else
    \begin{align*}
      C^\layer\geq C_{\Mcal^\layer} C_{A^\layer}>0.
    \end{align*}
    In particular, if this holds for the first $\ell=1,\ldots,N-1$ layers of the neural network, then we have, for all $f\in\Hcal$,
    \begin{align*}
      W_N(f)\leq \left(\norm{f}^2-\norm{A^{(0)}f}^2\right) \cdot \prod_{\ell=1}^{N-1} \left(1-C_{\Mcal^\layer} C_{A^\layer}\right).
    \end{align*}
  \end{corollary}
  \begin{proof}
    The proof is given in Appendix \ref{sec: app - cor: generic rate for discrete measure spaces}.
  \end{proof}
  \begin{remark}
    The assumptions of Corollary \ref{cor: generic rate for discrete measure spaces} imply that $\mu^\layer(M^\layer)<\infty$. %
    Also note that the derived lower bound is invariant under rescaling of the measure, which is reasonable as this  leaves the structure of the feature extractor invariant.
     
    Let us now continue the discussion from %
    above regarding the concreteness and applicability of our results. First, note that the constraint on the measure space is satisfied for the counting measure, in which case $C_{\Mcal^\layer}=1$. This may even be regarded as the most practically relevant measure space, having a feature engineer in mind who works with whatever feature extractor on a digital computer.
    Second, there is a whole class of feature extractors, for which there exists a constant eigenfunction, $\eta\equiv C_{A^\layer}^{1/2}=\mu^\layer(M^\layer)^{-1/2}$. Indeed, for scattering CNNs defined via diffusion operators as in \cite[Equ. (8)]{chew2024geometric}, the eigenfunction $\varphi_0$ of $\Lcal$ (in the notation of \cite{chew2024geometric}) is often constant \cite[Remark 3]{chew2024geometric}, %
    and we have $(1,\varphi_0)\in E^\layer$ in the context of the above corollary; see also \cite[Theorems 2, 5]{chew2024geometric} and \cite[Section 5.1]{mallat2012group}, which are based on similar assumptions.

    Finally, the last two arguments are precisely the reason why Corollary \ref{cor: generic rate for discrete measure spaces} qualitatively\footnote{There is a difference in the resulting basis of the exponential decay:  \cite{zou2020graph} guarantees $C^\layer\geq 2/\#G$, %
    while our result only implies the slightly weaker bound $C^\layer\geq 1/\#G$. This gap ultimately stems from a low-pass condition, which is assumed in \cite[Equ. (11)]{zou2020graph}, while our Corollary \ref{cor: generic rate for discrete measure spaces} does not require an analog thereof.} %
    recovers \cite[Proposition 3.3]{zou2020graph}, also implying exponential energy decay for arbitrary input signals of the graph scattering CNNs described in \cite{zou2020graph}. %
  \end{remark}
  \section{Information retention in scattering CNNs over LCA groups}\label{sec: Information retention in scattering CNNs over LCA groups}

\subsection{LCA group scattering}\label{sec: LCA group scattering}
  In this section, we demonstrate how structural information about the signal domain can be exploited to derive explicit estimates for the energy propagation in LCA group scattering CNNs. The tools for the following results rely on the properties of LCA groups and their Fourier analysis; a comprehensive overview of which can be found in, e.g., \cite{rudin2017fourier}. 
  
  Let $G$ be an LCA group, and let $\widehat{G}$ be its Pontryagin dual. Fix a Haar measure $\mu_G$ on $G$.
  For $f \in L^1(G)\cap L^2(G)$, we define the Fourier transform of $f$ by
  \[\Fcal f(\xi):=\widehat{f}(\xi):=\int_{G} f(x) \overline{\xi(x)} \dmuG (x), \quad \xi \in \widehat{G}.\]
  In the following, we always use the unique Haar measure $\mu_{\widehat{G}}$ on $\widehat{G}$ such that the Fourier transform extends unitarily to an operator $\mathcal{F}:L^2(G)\to L^2(\widehat{G})$.

  Throughout this section, let $\Psi\cup\{\chi\}\subseteq L^1(G)\cap L^2(G)$ be a semi-discrete Parseval frame, i.e., an at most countable family of convolution filters that satisfy %
  \begin{align}\label{ass: semi-discrete Parseval frame}
    \forall h \in L^2(G):~\norm{h*\chi}^2+\sum_{\psi\in \Psi} \norm{h*\psi}^2 =\norm{h}^2, %
  \end{align}
  which is equivalent to the Littlewood-Paley condition,
  \begin{align}\label{eq: Littlewood-Paley condition}
    |\widehat{\chi}(\xi)|^2 + \sum_{\psi \in \Psi} |\widehat{\psi}(\xi)|^2=1, \quad \text{a.e. } \xi \in \widehat{G}.
  \end{align}  
  For $g\in L^1(G)\cap L^2(G)$, let $\Ccal_g\in \Bcal(L^2(G))$ be the convolution operator $f \mapsto f*g.$
  The filters $\Psi\cup\{\chi\}$ induce a scattering CNN \cite{mallat2012group,bruna2013invariant}, with each layer sharing the same architecture, %
  \begin{align*}
    \Hcal^\layer:=L^2(G), \; \Lcal^\layer:=\set{\Ccal_\psi~\middle|~ \psi \in \Psi}, \; A^\player:=\Ccal_\chi,
  \end{align*}
  and the pointwise modulus $\sigma^\layer(f):=|f|$ for $f\in L^2(G)$ as its nonlinearity. By \eqref{ass: semi-discrete Parseval frame}, Assumption \ref{ass: Frame analog} and hence \eqref{eq: energy conservation} hold with equality. 
    By Proposition \ref{prop: energy conservation}, the scattering transform associated with these filters is a well-defined, norm-decreasing (nonlinear) operator $S:L^2(G)\to \ell^2\left(L^2(G)\right)$. %
  
    \subsection{A setting-specific rate for energy decay}\label{sec: A setting-specific rate for energy decay}
  
  The results in this section are based on scattering-specific ideas from \cite{mallat2012group,czaja2019analysis,zou2020graph}, which we combine with analogs for sumset-estimates based on ideas from \cite{ruzsa1999analog}, \cite[Section 2.4]{tao2006additive}.

  The modulus shifts part of the frequency content of a signal towards a neighborhood of $\mathds{1}_G\in \widehat{G}$. It is hence natural to assume that the output-generating filter $\chi$ satisfies a low-pass condition (which, by \eqref{eq: Littlewood-Paley condition}, is equivalent to a high-pass condition on the filters $\Psi$) with the intention that this makes the scattering CNN more informative \cite{mallat2012group,wiatowski2017energy,czaja2019analysis}. %
    \begin{assumption}\label{ass: low-pass/high-pass condition}
  	Let
  	$\Gamma_\Psi:=\widehat{G}\setminus \bigcup_{\psi \in \Psi} \{\xi \in \widehat{G}~|~\widehat{\psi}(\xi)\neq 0\}$. We assume that $\Gamma_\Psi$ is a neighborhood of $\mathds{1}_{G}\in\widehat{G}$. 
  \end{assumption}
  The next theorem confirms the intuition that the extracted features tend to carry more energy when the frequency gap $\Gamma_\Psi$ is large and the Fourier supports of $\Psi$ are uniformly small. Part of this stems from the low-pass filter $\chi$ capturing more signal information, though the full picture is more nuanced.

  \begin{theorem}\label{thm: rate for finite LCA group scattering}
   Let $G$ be compact. If there exists $S\in \N$ with $\#\supp(\widehat{\psi})\leq S$ ($\psi \in \Psi$), %
    then, for all $f \in L^2(G)$, $N\in \N$, 
    \begin{align*}%
      W_N(f)\leq \left(\norm{f}^2-\norm{f*\chi}^2\right) \cdot \left(1-1/S\right)^{N-1}. 
    \end{align*}
	If $G$ is finite, one can choose $S\leq \#G-\#\Gamma_\Psi<\infty$.
  \end{theorem}
  \begin{proof}
    The proof is given in Appendix \ref{sec: app - thm: rate for finite LCA group scattering}.
  \end{proof}
  \begin{remark}\label{rm: exp decay for nonabelian finite group scattering}
    The proof of Theorem \ref{thm: rate for finite LCA group scattering} makes use of the (accessible) Fourier analysis for abelian (compact) groups, whose analogs are more involving in the nonabelian case. If $G$ is finite nonabelian, Corollary \ref{cor: generic rate for discrete measure spaces} still ensures exponential energy decay, albeit with a slightly worse basis compared with the abelian case.  In fact, under the low-pass condition $|\int \chi \dmuG|^2=1$, taking $(\lambda,\eta)=(1,\mu_G(G)^{-1/2}\cdot \mathds{1}_G)$ in Corollary \ref{cor: generic rate for discrete measure spaces} yields $\eta*\chi*\chi^*= \eta\geq \mu_G(G)^{-1/2}$. Since $\mu_G(\{e_G\})/\mu_G(G)=1/\#G$, we have, for all $f \in \C^G$, $N \in \N$,
    \begin{align*}
      W_N(f)\leq \left(\norm{f}^2-\norm{f*\chi}^2\right) \cdot \left(1-1/\#G\right)^{N-1}.
    \end{align*}
  \end{remark}

  \begin{theorem}\label{thm: rate for general LCA group scattering}
    Let $G$ be any LCA group. There exists an open neighborhood $\Gamma=\Gamma^{-1}\subseteq \widehat{G}$ of $\mathds{1}_G\in \widehat{G}$ with $\Gamma^8\subseteq \Gamma_\Psi$. For any such $\Gamma$ we have, for all $f\in L^2(G)$, $N\in \N$,%
    \begin{align}\label{eq: upper bound for W_N, LCA scattering}
      W_N(f)\leq \left(\norm{f}^2-\norm{f*\chi}^2\right) \cdot \alpha(\Psi,\Gamma)^{N-1},
    \end{align}
    where 
    \begin{align*}
      \alpha(\Psi,\Gamma)=1-\frac{\muGhat{\Gamma^2}^2}{\Ncal(\Psi,\Gamma^2)\cdot \muGhat{\Gamma^4}^2} \in [0,1],
    \end{align*}
    and $\Ncal(\Psi,\Gamma^2)$ denotes the smallest number $n\in \N\cup \{\infty\}$ with the property that for any $\psi \in \Psi$ there exist $\xi_1,\ldots,\xi_n\in\widehat{G}$ such that $\supp(\widehat{\psi})\subseteq \bigcup_{k=1}^n \xi_k \Gamma^2$. 
    Furthermore, it holds that
    \begin{align}\label{eq: upper bound on covering number}
      \Ncal(\Psi,\Gamma^2) \leq \sup_{\psi \in \Psi} \left\lfloor \frac{\mu_{\widehat{G}}(\Gamma \cdot \supp(\widehat{\psi}))}{\muGhat{\Gamma}} \right\rfloor.
    \end{align}
    Thus, $\alpha(\Psi,\Gamma)<1$ if $\sup_{\psi\in\Psi}\mu_{\widehat{G}}(\Gamma \cdot \supp(\widehat{\psi}))<\infty$, i.e., we then have exponential energy decay globally on $L^2(G)$. In particular, if $G$ is discrete, exponential energy decay holds globally on $L^2(G)$.
  \end{theorem}
  \begin{proof}
    The proof is given in Appendix \ref{sec: app - thm: rate for general LCA group scattering}.
  \end{proof}
  \begin{remark}
    Theorem \ref{thm: rate for finite LCA group scattering}, Remark \ref{rm: exp decay for nonabelian finite group scattering}, and Theorem \ref{thm: rate for general LCA group scattering} corroborate the experimental reports (e.g., \cite{bruna2013invariant,anden2011multiscale}) that the energy decays exponentially with increasing network depth in digital implementations of scattering CNNs, which is even reinforced by their measure-theoretic analog, Corollary \ref{cor: generic rate for discrete measure spaces}. 
    However, the fact that the digital implementation of any feature extractor necessarily deals with discrete finite signals seems to play a significant role for its energy distribution. In fact, the recent paper \cite{fuhr2025energy} shows that the energy decay can even be arbitrarily slow in certain scattering CNNs over the Euclidean domain $\R^d$. Note that this is no contradiction to the findings in this paper, even when considering the digitial implementations of scattering CNNs as suitable discretizations of their Euclidean-domain counterparts: The basis of the exponential decay guaranteed by, e.g., Theorem \ref{thm: rate for finite LCA group scattering}, converges to $1$ if $\sup_{\psi\in\Psi}\#\supp(\widehat{\psi})$ increases at the order of $\#G$ as the size of the group tends to infinity.
  \end{remark}
  \begin{remark}
    Theorem \ref{thm: rate for general LCA group scattering} generalizes \cite[Proposition 3.3]{czaja2019analysis} (arbitrary LCA groups, explicit bound for the decay basis). %
  \end{remark}

  Our results contribute to the understanding of the energy distribution in deep feature extractors and, as a by-product, partially, positively answer the question of the applicability of a stability result for scattering CNNs posed in \cite[Remark 7]{chew2024geometric}.

\section*{Acknowledgment}

  The author would like to extend his heartfelt thanks to Hartmut Führ for insightful discussions and careful reading of the manuscript.

  \bibliographystyle{IEEEtran} 
  \bibliography{IEEEabrv,References}
  \cleardoublepage
\section{Appendix}\label{sec: app}

\subsection{Proof of Theorem \ref{thm: energy decay in the general setting}}\label{sec: app - thm: energy decay in the general setting}

The statement is trivial if $E^\layer=\emptyset$ or $\Rcal^\layer=\{0\}\subseteq \Hcal^\layer$. 
For the other case, choose any $0\neq g \in \Rcal^\layer$ and $(\lambda,\eta)\in E^\layer$. By \eqref{eq: introduction of iota_N}, it just remains to prove that $\iota_\ell\geq C^\layer$, $\ell \in \N$.
Since $\eta$ is normalized, there exists an orthonormal basis $\set{\eta}\dot{\cup} \Upsilon$ of $\Hcal^\layer\supseteq \Rcal^\layer$, which allows us to write 
\begin{align*}
	A^\layer g = \ip{g,\eta}A^\layer \eta + \sum_{\upsilon \in \Upsilon} \ip{g,\upsilon} A^\layer \upsilon,
\end{align*}
where the series converges unconditionally in $\Hcal^\layer$. Since $(\lambda,\eta)\in E^\layer$, we find that%
\begin{align}\label{eq: A^ell varphi and A^ell phi orthogonal}
	\ip{A^\layer\eta,A^\layer\upsilon}=\ip{{A^\layer}^\ast A^\layer\eta,\upsilon}=\ip{\lambda \eta,\upsilon}=0.
\end{align} 
Replacing $\upsilon$ by $\eta$ in \eqref{eq: A^ell varphi and A^ell phi orthogonal} also shows that $\norm{A^\layer \eta}^2=\lambda$.
Consequently, the Pythagorean theorem gives
\begin{align}\label{eq: lower bound for norm of A sigma g}
	\norm{A^\layer g}^2 &= \norm{\ip{g,\eta}A^\layer \eta}^2 
	+ \norm{\sum_{\upsilon \in \Upsilon} \ip{g,\upsilon} A^\layer \upsilon}^2 \nonumber\\
	&\geq \left|\ip{g,\eta}\right|^2 \cdot \lambda. %
\end{align} 
Taking the supremum over all pairs $(\lambda,\eta)\in E^\layer$ in \eqref{eq: lower bound for norm of A sigma g}, normalizing and inserting this into \eqref{def: iota_N} concludes the proof.

\subsection{Proof of Corollary \ref{cor: generic rate for discrete measure spaces}}\label{sec: app - cor: generic rate for discrete measure spaces}
Any $0\neq g\in\Rcal^\layer\subseteq \Rcal(\sigma^\layer)$ satisfies $g\geq 0$ ($\mu^\layer$-a.e.) by assumption on $\sigma^\layer$. Thus,
\begin{align}
	\left|\ip{g/\norm{g},\eta}\right|^2 \cdot \lambda %
	&= \left(\int \frac{g}{\norm{g}} \cdot \lambda^{1/2}\eta \dmul\right)^2 \nonumber\\
	&\geq C_{A^\layer} \cdot \norm{g}_{L^1(\Mcal^\layer)}^2/\norm{g}_{L^2(\Mcal^\layer)}^2 \nonumber\\
	&\geq C_{A^\layer} \cdot C_{\Mcal^\layer} , \label{eq: L1-L2 embedding}
\end{align}
where \eqref{eq: L1-L2 embedding} is due to the embedding $L^1(\Mcal^\layer) \hookrightarrow L^2(\Mcal^\layer)$, which itself is a well-known consequence from the assumption that $C_{\Mcal^\layer}>0$. Comparing with \eqref{def: constant C^ell}, we conclude
\begin{align*}
	C^\layer&\geq C_{A^\layer} \cdot C_{\Mcal^\layer}.
\end{align*}
The upper bound for $W_N(f)$ %
now follows from Theorem \ref{thm: energy decay in the general setting}.

\subsection{Proof of Theorem \ref{thm: rate for finite LCA group scattering}}\label{sec: app - thm: rate for finite LCA group scattering}
Any $g\in \Rcal^{(N)}$ can be written as $g=|f*\psi|$ for some $f\in L^2(G)$, $\psi \in \Psi$. Thus, 
\begin{align*}
	\iota_N\geq \inf_{f\in L^2(G), \psi \in \Psi:\norm{f*\psi}=1} \norm{|f*\psi|*\chi}^2.
\end{align*}
By \eqref{eq: introduction of iota_N}, it suffices to show that, for all $f\in L^2(G)$, $0\neq\psi \in \Psi$,
\begin{align}\label{eq: iota_N vs Gamma, finite group}
	\norm{|f*\psi|*\chi}^2 \geq (\#\supp(\widehat{\psi}))^{-1} \cdot \norm{f*\psi}^2.
\end{align}
The Littlewood-Paley condition \eqref{eq: Littlewood-Paley condition} entails that $|\widehat{\chi}|\equiv 1$ holds on $\Gamma_\Psi$. In particular, $|\widehat{\chi}|\geq \delta_{\mathds{1}_G}$ holds $\mu_{\widehat{G}}$-a.e., and hence, by Parseval's theorem and the convolution theorem, 
\begin{align}\label{eq: analog to Mallat's lemma/Czaja,Li for finite groups}
	\norm{|f*\psi|*\chi}^2=\norm{\Fcal(|f*\psi|)\cdot \widehat{\chi}}^2
	\geq \norm{\Fcal(|f*\psi|)\cdot \delta_{\mathds{1}_G}}^2. 
\end{align}
The RHS of \eqref{eq: analog to Mallat's lemma/Czaja,Li for finite groups} equals 
\begin{align*}
	\mu_{\widehat{G}}(\{\mathds{1}_G\})\cdot |\Fcal(|f*\psi|)(\mathds{1}_G)|^2= \mu_{\widehat{G}}(\{\mathds{1}_G\})\cdot\norm{f*\psi}_1^2,
\end{align*}
and we have $\Fcal(f*\psi)\equiv 0$ on $\widehat{G}\setminus \supp(\widehat{\psi})$, which entails that
\begin{align*}
	\norm{f*\psi}_1^2 \geq \norm{\Fcal(f*\psi)}_\infty^2 \geq \frac{\norm{\Fcal(f*\psi)}^2}{\mu_{\widehat{G}}(\{\mathds{1}_G\})\cdot \#\supp(\widehat{\psi})}.
\end{align*}
Finally, applying Parseval's theorem again, %
and putting all together, we conclude \eqref{eq: iota_N vs Gamma, finite group}; hence, the proof of the theorem.

\subsection{Proof of Theorem \ref{thm: rate for general LCA group scattering}}\label{sec: app - thm: rate for general LCA group scattering}
   
Similar to the Proof of Theorem \ref{thm: rate for finite LCA group scattering},
it suffices to show that, for all $f\in L^2(G)$, $\psi \in \Psi$,
\begin{align}\label{eq: iota_N vs Gamma}
	\norm{|f*\psi|*\chi}^2 \geq \frac{\muGhat{\Gamma^2}^2}{\Ncal(\Psi,\Gamma^2)\cdot \muGhat{\Gamma^4}^2} \cdot \norm{f*\psi}^2.
\end{align}

Let $\widehat{\phi}:=(\mathds{1}_{\Gamma^4}*\mathds{1}_{\Gamma^4})/\muGhat{\Gamma^4}$. This defines an auxiliary function %
$\phi\in L^2(G)$, which satisfies both $\phi\geq 0$ and $|\widehat{\chi}|\geq \widehat{\phi}$. By the symmetry of $\Gamma$, we further have %
\begin{align}\label{eq: pointwise lower bound for phi}
	\widehat{\phi}\big|_{\Gamma^2}\geq \frac{\muGhat{\Gamma^2}}{\muGhat{\Gamma^4}}.
\end{align}
Let $n:=\Ncal(\Psi,\Gamma^2)$. There are $\xi_1,\ldots,\xi_n\in \widehat{G}$ such that $\supp(\widehat{\psi})\subseteq \bigcup_{k=1}^{n} \xi_k \Gamma^2$. By Parseval's theorem and the convolution theorem, 
\begin{align}\label{eq: analog to Mallat's lemma/Czaja,Li , eq 1}
	\norm{|f*\psi|*\chi}^2=\norm{\Fcal(|f*\psi|)\cdot \widehat{\chi}}^2 
	\geq \norm{\Fcal(|f*\psi|)\cdot \widehat{\phi}}^2.
\end{align}
Employing \eqref{eq: pointwise lower bound for phi} and the nonnegativity of $\phi$, we obtain 
\begin{align}\label{eq: analog to Mallat's lemma/Czaja,Li , eq 2}
	\eqref{eq: analog to Mallat's lemma/Czaja,Li , eq 1}~ &=\norm{|f*\psi|*|\xi_k \cdot \phi|}^2 \nonumber\\
	&\geq \norm{f*\psi*(\xi_k \cdot \phi)}^2 \nonumber\\
	&\geq \int_{\xi_k \Gamma^2} |\widehat{f}(\xi)|^2 \cdot |\widehat{\psi}(\xi)|^2 \cdot |\widehat{\phi}(\xi_k^{-1}\cdot \xi)|^2 \dmuGhat(\xi) \nonumber\\
	&\geq \left(\frac{\muGhat{\Gamma^2}}{\muGhat{\Gamma^4}}\right)^2 \int_{\xi_k \Gamma^2} |\widehat{f}(\xi)|^2 \cdot |\widehat{\psi}(\xi)|^2 \dmuGhat(\xi).
\end{align}
Summing over $k$, and combining \eqref{eq: analog to Mallat's lemma/Czaja,Li , eq 1} and \eqref{eq: analog to Mallat's lemma/Czaja,Li , eq 2} yields
\begin{align*}
	&n \cdot \norm{|f*\psi|*\chi}^2 \\
	&\geq \left(\frac{\muGhat{\Gamma^2}}{\muGhat{\Gamma^4}}\right)^2 \int_{\bigcup_{k=1}^{n} \xi_k \Gamma^2} |\widehat{f}(\xi)|^2 \cdot |\widehat{\psi}(\xi)|^2 \dmuGhat(\xi) \\
	&\geq \left(\frac{\muGhat{\Gamma^2}}{\muGhat{\Gamma^4}}\right)^2 \int_{\supp(\widehat{\psi})} |\widehat{f}(\xi)|^2 \cdot |\widehat{\psi}(\xi)|^2 \dmuGhat(\xi). %
\end{align*}
Since the last integral equals $\norm{f*\psi}^2$, rearranging terms concludes the proof of \eqref{eq: iota_N vs Gamma}.

Finally, \eqref{eq: upper bound on covering number} follows from a straightforward generalization of Ruzsa's covering lemma for sumsets (cf. \cite{ruzsa1999analog} and \cite[Lemma 2.14]{tao2006additive}) to LCA groups. %

The existence of $\Gamma$ and the conclusion for $G$ discrete (i.e., $\widehat{G}$ compact) follow from the standard theory for LCA groups.

  \end{document}